\documentclass[11pt]{article}

\usepackage{mathtools,float, keyval}  
\usepackage{commath}              
\usepackage{amsmath,amsfonts}
\usepackage{amsthm}

\usepackage{latexsym}
\usepackage{algorithm}
\usepackage{algpseudocode}
\usepackage{algorithmicx}
\usepackage[margin=2cm]{geometry}
\usepackage{setspace}
\usepackage{caption}
\usepackage{multirow}
\usepackage{subcaption}
\usepackage{lastpage}
\usepackage{float}
\usepackage{wrapfig}
\usepackage{adjustbox}
\usepackage{hyperref}
\usepackage{xcolor}
\usepackage{xparse}
\usepackage{lineno}
\usepackage{subcaption}
\usepackage{cleveref}
\usepackage{placeins}
\usepackage{textcomp}
\usepackage{xparse}
\usepackage{hyperref}                 
\usepackage{authblk}       			  
\usepackage{placeins}
\usepackage[english]{babel}

\newtheorem{theorem}{Theorem}[section]
\newtheorem{lemma}[theorem]{Lemma} 

\newtheorem{proposition}[theorem]{Proposition}

\newtheorem{remark}{Remark}[section]

\DeclareMathOperator{\tr}{tr}

\DeclareMathOperator{\diag}{diag}

\begin{document}

\title{ROIPCA: An online memory-restricted PCA algorithm based on rank-one updates}

\author[1]{Roy Mitz\thanks{roymitz@mail.tau.ac.il} }
\author[1]{Yoel Shkolnisky\thanks{yoelsh@post.tau.ac.il}}
\affil[1]{\footnotesize{School of Mathematical sciences, Tel-Aviv University, Tel-Aviv, Israel}}
\date{}

\maketitle

\begin{abstract}

Principal components analysis (PCA) is a fundamental algorithm in data analysis. Its memory-restricted online versions are useful in many modern applications, where the data are too large to fit in memory, or when data arrive as a stream of items. In this paper, we propose ROIPCA and fROIPCA, two online PCA algorithms that are based on rank-one updates. While ROIPCA is typically more accurate, fROIPCA is faster and has comparable accuracy. We show the relation between fROIPCA and an existing popular gradient algorithm for online PCA, and in particular, prove that fROIPCA is in fact a gradient algorithm with an optimal learning rate. We demonstrate numerically the advantages of our algorithms over existing state-of-the-art algorithms in terms of accuracy and runtime.
\end{abstract}


\section{Introduction}

Principal components analysis (PCA)~\cite{jolliffe2016principal} is a popular method in data science. Its main objective is reducing the dimension of data, while preserving most of their variance. Specifically, PCA finds the orthogonal directions in space where the data exhibit most of their variance. These directions are called the principal components of the data. In the classical setting, also called batch PCA or offline PCA, the input to the PCA algorithm is a set of vectors $\{ x_i \}_{i=1}^n \in \mathbb{R}^d$, which are centred to have mean zero in each coordinate. PCA then solves the following optimization problem. The first principal component, denoted here by $v_1 \in \mathbb{R}^d$, is the solution to
\begin{equation} \label{pca}
v_1 = \operatornamewithlimits{argmax}_{\norm{v} = 1} \sum_{i=1}^{n} (v^T x_i)^2.
\end{equation}
The other components are obtained iteratively, similarly to~\eqref{pca}, while additionally requiring orthogonality of each component to the components already calculated, namely,
\begin{equation} \label{pca2}
v_k = \operatornamewithlimits{argmax}_{\substack{\norm{v} = 1 \\ v \perp v_1,\ldots,v_{k-1}}} \sum_{i=1}^{n} (v^T x_i)^2.
\end{equation}

Examples for applications of PCA include algorithms for clustering such as k-means \cite{arthur2007k}, regression algorithms such as linear regression~\cite{seber2012linear} and SVR~\cite{smola2004tutorial}, and classification algorithms, such as SVM~\cite{burges1998tutorial} and logistic regression~\cite{menard2002applied}. Additionally, PCA is known to denoise the data, which is by itself a good reason to use it as a pre-processing step for the data before any further analysis.

Classical PCA is typically implemented via the eigenvalue decomposition of the sample covariance matrix of the data, or the singular value decomposition (SVD) of the centred data itself. These decompositions are computed using algorithms that generally require $O(nd\min (n,d))$ floating point operations, where~$n$ is the number of data points and~$d$ is the dimension of the data. In many cases, however, the data are low-rank and the number of required principal components is much smaller than~$d$. If~$m$ is the number of required principal components, approximate PCA decompositions, using algorithms such as the truncated SVD, can be computed in $O(ndm)$ floating point operations~\cite{halko2011algorithm,halko2011finding}. 

In terms of memory requirements, all batch PCA algorithms require the entire dataset to fit either in the random-access memory (RAM) or in an external storage, such as a disk drive, resulting in $O(nd)$ space complexity. Algorithms based on the covariance matrix require an additional $O(d^2)$ space to store it. When $d$ is large, it may be infeasible to store the $O(d \times d)$ covariance matrix. Because of its time and space complexity, batch PCA is essentially infeasible for large datasets. In the big data era, fast and accurate alternatives to the classical batch PCA algorithms are essential.

In the memory-restricted online PCA setting, the vectors $\{ x_i \}_{i=1}^n \in \mathbb{R}^d$ are presented to the algorithm one by one. For each vector presented, the algorithm updates its current estimate of the principal components without recalculating them from scratch. In particular, the memory requirements of the algorithm are not allowed to grow with $n$, and are usually restricted to only $O(md)$, which is proportional to the memory required to store the principal components themselves. 

The online setting is useful in various scenarios. Examples for such scenarios include cases where the data are too large to fit memory, when working with data streams whose storage is not feasible, or if the computation time of batch PCA for each new point is too long for the task at hand.

Online PCA algorithms usually start with an initial dataset $X_0 \in \mathbb{R}^{n_0 \times d}$, where rows correspond to data points and columns to features. The principal components of this dataset are calculated by any batch PCA algorithm and are used as the basis for the online algorithm. Then, the online algorithm processes new data points one at a time and updates the principal components accordingly. Typically, online PCA algorithms only estimate the few top principal components of the data.

Many algorithms for online PCA were proposed in the literature. A survey of some of these algorithms is given in~\cite{cardot2018online}. There are four classes of algorithms for online PCA. The first and broadest class of algorithms is based on stochastic gradient descent (SGD). The classical SGD algorithm for online PCA is called Oja's algorithm~\cite{oja1982simplified,oja1992principal}, and most SGD online PCA methods are based on it. A related algorithm, called the generalized Hebbian algorithm (GHA)~\cite{sanger1989optimal}, performs essentially the same updates as the SGD methods, but differs in its re-orthogonalization process, see~\cite{cardot2018online} for details. Computationally, each iteration of the above methods requires $O(md)$ arithmetic operations, which makes them computationally attractive.

All SGD-based PCA algorithms require a learning rate parameter. Some methods such as CCIPCA~\cite{weng2003candid} and AdaOja~\cite{henriksen2019adaoja} auto-tune it, but it is usually determined heuristically~\cite{lv2006global, darken1992learning} via some form of a grid-search. The asymptotic convergence properties of SGD-based methods depend on the learning rate, and are well understood for the case $m = 1$~\cite{oja1982simplified}. Several works of a  theoretical nature consider the problem of determining the learning rate and provide some guarantees for that case, for example~\cite{allen2017first,jain2016streaming,shamir2015stochastic, shamir2016convergence} and the references therein. Many of them rely on information that is not available in the online setting, such as the eigengap of the covariance matrix.  

The second, and more recent, class of methods for online PCA is essentially an adaptation of SGD-based methods to update the principal components over blocks of data, rather than one sample at a time~\cite{mitliagkas2013memory, li2016rivalry, yang2018history}. The original block method~\cite{mitliagkas2013memory} is proved to converge under some data distributions, but the convergence depends on the block size. The methods of this class thus require setting a block size parameter, which may not be an easy task.

The third class of online PCA methods is based on the assumption that the data are low-rank. Under that assumption, the IPCA algorithm~\cite{arora2012stochastic} updates the principal components using an exact closed-form formula which does not require any parameter tuning.  While providing good approximation to the principal components (as we shall see later), this method's time complexity is $O(m^2d)$.

The fourth and most recent class of methods is based on matrix sketching. The works of~\cite{karnin2015online} and \cite{ghashami2016frequent} provide theoretical analysis for their online approximation schemes, but these methods bear little practical use.

Online kernel PCA has also gained attention recently. While not being the main focus of this paper, we wish to outline the main works of this research field. Most online kernel PCA algorithms adjust any of the aforementioned algorithms to the kernel settings, by applying them on the relevant kernel matrix rather than on the sample covariance matrix. The algorithms proposed in~\cite{gunter2007fast, honeine2011online} are adaptations of the generalized Hebbian algorithm to the kernel settings. Other gradient-based adaptations are~\cite{zhang2016stochastic, xie2015scale}. Other works, such as~\cite{chin2006incremental, chin2007incremental, takeuchi2007efficient, tokumoto2011fast, joseph2016online, zhao2019two} essentially use the IPCA algorithm's updating scheme. Methods based on sketching~\cite{ghashami2016streaming} and random features~\cite{ullah2018streaming} were also proposed. We will later describe briefly how our proposed methods can be adjusted to the online kernel PCA setting as well. While kernel PCA with a linear kernel is essentially just a regular PCA, it is inappropriate for the memory-restricted online setting, as the dimension of the matrices involved in kernel methods is equal to the number of samples rather than the number of features. We will thus not include online kernel PCA methods in the numerical section.

In this paper, we propose two approaches for online PCA. The first one, named ROIPCA, is based on a closed-form formula for eigenvectors update, and is thus exact under certain assumptions. The second approach, named fROIPCA, is approximate, but faster and is proven to be optimal in some sense. More concretely, we prove that fROIPCA is in fact an SGD-based method with an optimal learning rate in the sense that it minimizes the error obtained after one iteration. To the best of our knowledge, this is the first practical gradient method which is proved to be optimal is some sense. Contrary to most other online PCA methods, both of our approaches do not require any parameter tuning. They also meet the $O(md)$ memory constraint of the restricted-memory setting. We show using simulations that our approaches are more accurate than the existing ones, and are usually faster.
 
The rest of this paper is organized as follows. In Section~\ref{sec:rank_one}, we describe the rank-one update problem, which is used as a building block in our algorithms. In Sections~\ref{sec:our_algo} and~\ref{sec:froipca}, we present the ROIPCA and fROIPCA algorithms, respectively. In Section~\ref{sec:discussion}, we discuss several implementation details. In Section~\ref{sec:numerical}, we illustrate numerically the advantages of our approach for both synthetic and real data.

\section{Rank-one updates}  \label{sec:rank_one}

We denote by $X = [x_1x_2 \cdots x_d]$ a matrix whose columns are $x_{1},\ldots,x_{d}$, and by $X^{(m)} = [x_1\cdots x_m]$ its truncated version consisting only of its first $m$ columns, $m \leq d$. Let $A$ be a $d \times d$ real symmetric matrix with real (not necessarily distinct) eigenvalues $\ell_1 \ge \ell_2 \ge \cdots \ge \ell_d$ and associated orthonormal eigenvectors $u_1,\ldots,u_d$. Using this notation, the eigendecomposition of $A$ is given by $A = U\Lambda U^T$, with $U = [u_1u_2 \cdots u_d]$ and $\Lambda = \diag (\ell_1, \ldots,\ell_d)$. The problem of (symmetric) rank-one update is to find the eigendecomposition of
\begin{equation} \label{eq:prob}
A + \rho vv^T, \quad \rho \in \mathbb{R}, \quad  v \in \mathbb{R}^d,  \quad  \norm{v} = 1.
\end{equation}
We denote the eigenvalues of~\eqref{eq:prob} by $\gamma_1 \ge \gamma_2 \ge \cdots \ge \gamma_d$ and their associated orthogonal eigenvectors by $s_1,\ldots,s_d$, that is $A + \rho vv^T = S\Gamma S^T$, with $S = [s_1s_2 \cdots s_d]$ and $\Gamma = \diag (\gamma_1, \ldots,\gamma_d)$.
The relation between the eigendecompositions before and after a rank-one update is well-studied, e.g., \cite{bunch1978rank, ding2007eigenvalues}. Without loss of generality, we further assume that $\ell_1 > \ell_2 > \cdots > \ell_d$ and that for $z = U^Tv$ we have $z_j \neq 0$ for all $1 \leq j \leq d$. The deflation process in~\cite{bunch1978rank} transforms any update~\eqref{eq:prob} to satisfy these assumptions.  Given the eigendecomposition $A = U\Lambda U^T$, the updated eigenvalues $\gamma_1,\ldots,\gamma_d$ of $A + \rho v v^T$ are given by the $d$ roots of the secular equation~\cite{bunch1978rank}
\begin{equation} \label{eqn:secular_equation}
w(t) = 1 + \rho \sum_{k=1}^{d}\frac{z_k^2}{\ell_k - t} , \quad  z = U^Tv.
\end{equation}
The eigenvector corresponding to the $i$-th root (eigenvalue)~$t_i$ is given by the explicit formula
\begin{equation}  \label{eqn:EigenvaectorFormula}
s_i = \frac{\tilde{s}_i}{\norm{\tilde{s}_i}} , \quad \tilde{s}_i = U\Delta_i^{-1}z,
\quad  z = U^Tv , \quad  \Delta_i = \Lambda - \gamma_i I,
\end{equation}
which is equivalent to
\begin{equation}  \label{eq:formula_as_sum}
\tilde{s}_i = \sum_{k=1}^{d} \frac{\langle u_k, v \rangle}{\ell_k - \gamma_i}u_k .
\end{equation}

\subsection{Rank-one update with partial spectrum}  \label{sec:ro_partial}
Formulas~\eqref{eqn:secular_equation} and~\eqref{eqn:EigenvaectorFormula} assume that the entire spectrum of the matrix $A$ is known. In various settings, as we will see in the next section, only the top~$m$ eigenvalues and eigenvectors of~$A$ are known, making the classical formulas inapplicable in these cases. The authors in~\cite{mitz2019symmetric} provide adjusted formulas for rank-one updates with partial spectrum, summarized in the following propositions.

\begin{proposition}[Rank-one update with partial spectrum \cite{mitz2019symmetric}] \label{prop:first_order}
Let $\mu \in \mathbb{R}$ be a fixed parameter (described below). The eigenvalues of~\eqref{eq:prob} are approximated by the roots of the first order truncated secular equation
\begin{equation} \label{eqn:TSE}
w_{1}(t ; \mu) = 1 + \rho \left( \sum_{k=1}^{m} {\frac{z_k^2}{\ell_k - t}} + \frac{1 - \sum_{k=1}^{m}{z_k^2}}{\mu - t} \right).
\end{equation}
The error in each approximated eigenvalue is of magnitude $ O \big( \max_{m+1 \leq j \leq d} \abs{\ell_j - \mu} \big) $. Furthermore, the eigenvectors of~\eqref{eq:prob} are approximated by the formula
\begin{equation}\label{eqn:TEF}
\begin{aligned} 
\hat{s}_i &= U^{(m)}(\Delta_i^{(m)})^{-1}(U^{(m)})^Tv + \frac{1}{\mu - \gamma_i}r,  \\ \quad  r &= v - U^{(m)} (U^{(m)})^Tv ,
\end{aligned}
\end{equation}
for $1 \leq i \leq m$. The error in each approximated eigenvector is of magnitude $ O \big( \max_{m+1 \leq j \leq d} \abs{\ell_j - \mu} \big) $.
\end{proposition}
Formulas~\eqref{eqn:TSE}--\eqref{eqn:TEF} are based on a parameter $\mu$. Several options were proposed in~\cite{mitz2019symmetric} for choosing~$\mu$. When the data are known to be low-rank, choose $\mu = 0$. Otherwise, choose
\begin{equation} \label{eq:mu_mean}
\mu_{mean} = \frac{\sum_{k=m+1}^d \ell_k}{d-m} = \frac{\operatorname{tr}(A) - \sum_{k=1}^{m} \ell_k}{d - m} ,
\end{equation}
which is the mean of the unknown eigenvalues. For more details see~\cite{mitz2019symmetric}. In Section~\ref{sec:choose_mu}, we will discuss the choice of $\mu$ that is most suitable to our setting. There exist second order formulas for rank-one update, which are more accurate~\cite{mitz2019symmetric}. However, their implementation in the online PCA setting will require the storage of the entire covariance matrix, and thus these formulas are less suitable for this work.

\subsection{Fast rank-one update with partial spectrum}  \label{sec:fast_update}

The most computationally expensive part of the rank-one update procedure described in the previous section is the eigenvectors formula~\eqref{eqn:TEF}, which requires $O(m^2d)$ floating point operations. In this section, we introduce a new rank-one update procedure for the eigecnvectors, whose complexity is only $O(md)$ operations. We demonstrate in Section~\ref{sec:numerical} that this faster procedure maintains accuracy similar to the procedure described in Section~\ref{sec:ro_partial}. This approximation formula is summarized in the following proposition.

\begin{proposition}[Fast rank-one update with partial spectrum]\label{prop:fast extension}
Let $1 \leq i \leq m$ and let $\tau_i \in \mathbb{R}$ be a fixed parameter. The eigenvectors of~\eqref{eq:prob} are approximated by the fast first order eigenvectors formula
\begin{equation}\label{eqn:fast_TEF}
\begin{aligned} 
\hat{s}_i &= \bigg( \frac{1}{\ell_i - \gamma_i} - \tau_i \bigg)  \langle u_i , v  \rangle u_i + \tau_i (v - r) + \frac{1}{\mu - \gamma_i}r , \\
r &= v - U^{(m)} (U^{(m)})^Tv ,
\end{aligned}
\end{equation}
with an error term
\begin{equation} \label{eqn:err_term_tau}
O \bigg(\max_{1 \leq j \leq m, j \neq i} \abs{\frac{1}{\ell_j - \gamma_i} - \tau_i} + \max_{m+1 \leq j \leq d} \abs{\ell_j - \mu} \bigg) .
\end{equation}
\end{proposition}

The proof of Proposition~\ref{prop:fast extension} is given in Appendix~\ref{app1}.

The error term~\eqref{eqn:err_term_tau} is an upper bound for the error. In Appendix~\ref{app1}, we obtain a tighter bound that is minimised when $\tau_i =\tau^*_i $, with
\begin{equation} \label{eqn:best_tau}
\tau^*_i = \frac{\sum_{k=1 , k\neq i}^{m} \frac{z_k^2}{\ell_k - \gamma_i}}{\sum_{k=1 , k\neq i}^{m} z_k^2}.
\end{equation}

\section{Online PCA based on rank-one updates}\label{sec:Online PCA based on rank-one updates}

We assume that the input data points are already centred.  If that is not the case, we refer the reader to~\cite{honeine2011online} for an online centering procedure. Our algorithms are based on the following observation. Let $X_n$ be an $n \times d$ matrix whose rows correspond to data points. Let $x_{n+1} \in \mathbb{R}^d$ be a new data point, and denote by $X_{n+1}$ the matrix whose upper $n\times d$ submatrix is $X_n$ and whose last row is $x_{n+1}$. Denote by $\hat{x}_{n+1}$ the normalized $x_{n+1}$, i.e., $\hat{x}_{n+1} = \frac{x_{n+1}}{\norm{x_{n+1}}}$. Then, 
\begin{align} \label{eq:XTX}
X_{n + 1}^TX_{n + 1} &= X_{n}^TX_{n} + x_{n+1}x_{n+1}^T \\
&= X_{n}^TX_{n} + \norm{x_{n+1}}^2 \hat{x}_{n+1} \hat{x}_{n+1}^T.
\end{align}
Equivalently, recalling that
\begin{equation} \label{eq:XTXs}
\text{cov}(X_{n}) = \frac{1}{n}X_{n}^TX_{n} ,
\end{equation}
we have that
\begin{equation} \label{eq:covs}
\text{cov}(X_{n+1}) = \frac{n}{n+1} \bigg( \text{cov}(X_{n}) + \frac{\norm{x_{n+1}}^2}{n}\hat{x}_{n+1} \hat{x}_{n+1}^T \bigg) .
\end{equation}
Here, the covariance matrix of the data, $\text{cov}(X_{n})$, replaces~$A$ of~\eqref{eq:prob}, $\frac{\norm{x_{n+1}}^2}{n}$ replaces $\rho$, and $\hat{x}_{n+1}$ replaces~$v$. We conclude that introducing a new data point to the covariance matrix is, up to a multiplicative constant, a rank-one update to the original covariance matrix. The proof of~\eqref{eq:covs} is straightforward. Thus, in order to compute the PCA in an online fashion, one can use any of the rank-one update formulas from Section~\ref{sec:rank_one}.

A common assumption in the literature is that high-dimensional data lie in a lower dimensional subspace. Therefore, similarly to the IPCA algorithm~\cite{arora2012stochastic}, in subsequent analysis, we assume that the data are of rank $m < d$, and compute only the~$m$ leading principal components of the data. We denote by $(t_i^n, p_{i}^{n})$ the $i$'th eigenpair of $X_n^TX_n$ for $1 \leq i \leq m$. Analogously, We denote by $(\lambda_i^n, q_{i}^{n})$ the approximated $i$'th eigenpair of $X_n^TX_n$ for $1 \leq i \leq m$.

In the following sections, we present our two online PCA algorithms based on rank-one updates, termed ROIPCA and fROIPCA.
 
\subsection{ROIPCA}  \label{sec:our_algo}

The ROIPCA algorithm is derived by using the first order approximations~\eqref{eqn:TSE} and~\eqref{eqn:TEF} to update the principal components, and by choosing the parameter~$\mu$ to be either $\mu = 0$ or $\mu=\mu_{mean}$ (see~\eqref{eq:mu_mean}). The update step of ROIPCA is

\begin{equation} \label{eqn:ROIPCA_update}
     q_i^{n+1} = \sum_{k=1}^{m} \frac{\langle q_k^{n} , x_{n+1} \rangle}{\lambda_k^{n} - \lambda_i^{n + 1} } q_k^n + \frac{1}{\mu - \lambda_i^{n + 1}} \Bigg( x_{n+1} - \sum_{k=1}^{m} \langle q_k^{n} , x_{n+1} \rangle q_k^{n} \Bigg),
\end{equation}

\noindent followed by normalization. The approximated eigenvalues $\{ \lambda_i^{n+1} \}_{i=1}^m$ required for this step are calculated by finding the roots of~\eqref{eqn:TSE}.

ROIPCA is covariance free, that is, does not require to store in memory the covariance matrix of the data. The ROIPCA algorithm is summarized in Algorithm~\ref{alg:onlne_pca}.

\begin{algorithm}
\caption{ROIPCA and fROIPCA- covariance free rank-one online PCA}
\label{alg:onlne_pca}
\begin{algorithmic}[1]
\Require $m$ leading eigenpairs $\{(\lambda_i,q_i)\}_{i=1}^m$ of $X_0^TX_0$, where $X_0$ is the initial dataset.
\Ensure An approximation $ \{(\widetilde{\lambda}_i,\widetilde{q}_i)\}_{i=1}^m$ of the eigenpairs of $X_n^TX_n$, where $X_n$ is the data matrix after $n$ updates.
\State Set $B = \tr{ \big( X_0^TX_0 \big) }$
\For { $\{{x_i}\}_{i=1}^n$ }
\State Set $\rho = \norm{x_i}^2$ \qquad \Comment{See~\eqref{eq:prob}}
\State Set $v = x_i / \norm{x_i}$ \qquad \Comment{See~\eqref{eq:prob}}
\State Set $z = Q^Tv$
\State Choose a parameter $\mu$ (i.e., $\mu=0$ or $\mu_{mean}$\eqref{eq:mu_mean})
\If{$\mu = \mu_{mean}$}
\State $\mu = \frac{B - \sum_{i=1}^{m}\lambda_i}{d - m}$
\State $B = B + \norm{x_i}^2$
\EndIf
\State Calculate the $m$ largest roots  $\{(\widetilde{\lambda}_i\}_{i=1}^m$ of the truncated secular equation \eqref{eqn:TSE} \\
\begin{equation} \label{eqn:TSE_alg}
w_{1}(t ; \mu) = 1 + \rho \left( \sum_{k=1}^{m} {\frac{z_k^2}{\lambda_k - t}} + \frac{1 - \sum_{k=1}^{m}{z_k^2}}{\mu - t} \right).
\end{equation}
\For { $\{{q_i}\}_{i=1}^m$ }
\State  For ROIPCA: find $\widetilde{q}_i$ using~\eqref{eqn:ROIPCA_update} \\
\begin{equation} \label{eqn:ROIPCA_update_alg}
      \widetilde{q}_i = \sum_{k=1}^{m} \frac{\langle q_k , v \rangle}{\lambda_k - \tilde{\lambda}_i} q_k + \frac{1}{\mu - \tilde{\lambda}_i} \Bigg( v - \sum_{k=1}^{m} \langle q_k , v \rangle q_k \Bigg) .
\end{equation}
\State  For fROIPCA: find $\widetilde{q}_i$ using~\eqref{eqn:froipca_update} \\
\begin{equation} \label{eqn:froipca_update_alg}
\widetilde{q}_i = q_i + \frac{1}{ \langle q_i , v \rangle^2}\bigg(  \frac{\lambda_i -\tilde{\lambda}_i}{\mu - \tilde{\lambda}_i} \bigg) \left( \langle q_i , v \rangle v -\langle q_i , v \rangle \sum_{k=1}^m \langle q_k , v \rangle q_k \right ) .
\end{equation}
 \EndFor
\State Set $ \tilde{q}_i = \frac{\tilde{q}_i}{\norm{\tilde{q}_i}}$ for $1 \leq i \leq m$
\State Set
$q_i = \tilde{q}_i$ and $\lambda_i = \tilde{\lambda}_i$ for $1 \leq i \leq m$
\EndFor
\State \Return $ \{(\widetilde{\lambda}_i,\widetilde{q}_i)\}_{i=1}^m$
\end{algorithmic}
\end{algorithm}

We now discuss the error of the algorithm proposed in this section. When the sample covariance matrix is low-rank (and practically, "close" to being low-rank), it is clear by the error terms in Proposition~\ref{prop:first_order} that ROIPCA is accurate for all proposed choices of~$\mu$. The low-rank assumption is relevant for many datasets where the intrinsic dimension of the data is much lower than the number of features. This assumption is also used in the theoretical analysis of other state-of-the-art algorithms for online PCA such as IPCA.

When the data are not low-rank, the error analysis of ROIPCA is more complicated, and is only given for one step in Proposition~\ref{prop:first_order} above. To the best of our understanding, no other online PCA algorithm accounts for this case. While we also do not provide an asymptotic error analysis for this case, we note that ROIPCA is the only algorithm with a mechanism for accounting for the unknown eigenvalues via the parameter $\mu$. Furthermore, we demonstrate numerically in Section~\ref{sec:num_acc} that when the data are not low-rank, the introduction of $\mu$ improves performance significantly compared to all other existing methods.

\subsection{fROIPCA}  \label{sec:froipca}

While ROIPCA provides excellent accuracy, its runtime is $O(m^2d)$ per iteration (see Section~\ref{sec:complexity}). This runtime is asymptotically the same as that of the IPCA algorithm~\cite{arora2012stochastic}, but is slower than SGD-based methods whose runtime is $O(md)$ per iteration.  In this section, we derive a faster version of ROIPCA, termed fROIPCA, and show that it can be regarded as an SGD-based method with an optimal learning rate in a sense defined below.

\subsubsection{Deriving an optimal learning rate}

The general Hebbian algorithm for online PCA, a classical SGD method, is given by the update rule~\cite{sanger1989optimal}
\begin{equation} \label{eqn:general_grad}
q_i^{n+1} = q_i^n + \eta_n \left( \langle q_i^{n} , x_{n+1} \rangle x_{n+1} - \langle q_i^{n} , x_{n+1} \rangle \sum_{k=1}^{i} \langle q_k^{n} , x_{n+1} \rangle q_k^n \right ),
\end{equation}
followed by normalization, where $\eta_n > 0$ is a learning rate which is typically determined heuristically. Instead, we suggest to use all~$m$ known eigenvectors in an update rule of the form
\begin{equation} \label{eqn:general_grad_ext}
q_i^{n+1} = q_i^n + \eta_n \left( \langle q_i^{n} , x_{n+1} \rangle x_{n+1} - \langle q_i^{n} , x_{n+1} \rangle \sum_{k=1}^m \langle q_k^{n} , x_{n+1} \rangle q_k^n \right ),
\end{equation}
\noindent followed by normalization.

\begin{remark} \label{rem:scale_not_imp}
In the analysis to follow, we will consider several update rules of the form of~\eqref{eqn:general_grad_ext}. Since the scale of the resulting vector is arbitrary, we will scale all update rules of this form so that the $q_i^n$ component that is not affected by the learning rate will have $1$ as its coefficient.  

\end{remark}

While the update rule~\eqref{eqn:general_grad_ext} is not of the form of SGD-based methods~\eqref{eqn:general_grad}, as it uses all~$m$ eigenvectors, we next prove that it is  superior to them for all choices of $\eta_n > 0$, as stated by the following proposition.

\begin{proposition} \label{prpo:better}
Let $1 \leq i \leq m$. Assume that the approximated eigenvectors at step $n$ are accurate, i.e., that $q_k^n = p_k^n$ for $1 \leq k \le m$. Denote by $w_i^{n+1}$ the accurate $i$'th eigenvector in step $n+1$ scaled according to Remark~\ref{rem:scale_not_imp}. Denote by~$u_i^1$ the result of evaluating the right hand side of~\eqref{eqn:general_grad}, and by~$u_i^2$ the result of evaluating the right hand side of~~\eqref{eqn:general_grad_ext}. Then, $\norm{w_i^{n+1} - u_i^2} \leq \norm{w_i^{n+1} - u_i^1}$ for all $\eta_n > 0$.
\end{proposition}

The proof of Proposition~\ref{prpo:better} is given in Appendix~\ref{app2}.

Furthermore, the update rule~\eqref{eqn:general_grad_ext} enables us to derive an optimal choice for $\eta_n$ for a single update step, as stated by the following proposition.

\begin{proposition}{(Optimal learning rate for SGD-methods)} \label{prpo:the_best}
Assume that $X_n^TX_n$ is of rank $m<d$ and let $1 \leq i \leq m$. Assume that $q_k^n = p_k^n$  and $\lambda_k^n = t_k^n$ for $1 \leq k \le m$. The learning rate $\eta_n$ in~\eqref{eqn:general_grad_ext} that minimizes $\norm{p_i^{n+1} - q_i^{n+1}}$ is
\begin{equation} \label{eqn:opr_eta}
\eta^*_n = \frac{1}{\langle q_i^{n} , x_{n+1} \rangle^2}\bigg( 1 - \frac{t^n_i}{t^{n+1}_i} \bigg).
\end{equation}
\end{proposition}

The proof of Proposition~\ref{prpo:the_best} is given in Appendix~\ref{app3}.

\noindent In Section~\ref{sec:num_opt} we demonstrate Proposition~\ref{prpo:the_best} numerically.

\subsubsection{Asymptotic behaviour of the optimal learning rate}

A sufficient condition for the convergence of the generalized Hebbian algorithm to the principal components of the data is~\cite{balsubramani2013fast}
\begin{equation} \label{eqn:suf_cond}
    \sum_{n=1}^{\infty} \hat{\eta}_n = \infty \quad \text{and} \quad \sum_{n=1}^{\infty} \hat{\eta}_n^2 < \infty .
\end{equation}

Consequently, most authors suggest a learning rate $\hat{\eta}_n \propto \frac{1}{n}$, which somewhat artificiality fulfils these criteria, and needs to be manually tuned to improve performance. Our suggested learning rate~\eqref{eqn:opr_eta} arises naturally from data-dependent quantities, does not require any tuning, and was proven in Proposition~\ref{prpo:the_best} to be optimal in some sense. We now prove that under mild assumptions, the learning rate~\eqref{eqn:opr_eta}  satisfies \eqref{eqn:suf_cond}.

\begin{proposition}{(Asymptotic behaviour of the optimal learning rate)} \label{prop:froipca_conv}
Let $\{ x_1, x_2, ...\} \in \mathbb{R}^d$ be drawn i.i.d. from some distribution with mean $\mu_0$ and a covariance matrix $\Sigma$, so that $\norm{x_n} = O_p(1)$ for all $n \in \mathbb{N}$, where $O_p(\cdot)$ is big-O notation in probability. Assume further that for all $n \in \mathbb{N}, 1 \leq i \leq m$, $\abs{ \langle q_i^n , x_{n+1} \rangle} \geq c$ for some constant $c > 0$. Then, the optimal learning rate~\eqref{eqn:opr_eta} satisfies~\eqref{eqn:suf_cond} with high probability.
\end{proposition}

The proof of Proposition~\ref{prop:froipca_conv} is given in Appendix~\ref{ap_conv}.

\subsubsection{The fROIPCA algorithm}

We now present the update rule of the fROIPCA algorithm. The fROIPCA algorithm is derived using the fast first order approximation~\eqref{eqn:fast_TEF} to update the principal components, by setting $\tau_i = 0$ and $v=x_{n+1}$, and by choosing the parameter~$\mu$ to be either $\mu = 0$ or $\mu=\mu_{mean}$ (see~\eqref{eq:mu_mean}). Explicitly, following Remark~\ref{rem:scale_not_imp}, we obtain the update rule
\begin{equation} \label{eqn:froipca_update}
q_i^{n+1} = q_i^n + \frac{1}{ \langle q_i^{n} , x_{n+1} \rangle^2}\bigg(  \frac{\lambda^n_i - \lambda^{n+1}_i}{\mu - \lambda^{n+1}_i} \bigg) \left( \langle q_i^{n} , x_{n+1} \rangle x_{n+1} - \langle q_i^{n} , x_{n+1} \rangle \sum_{k=1}^m \langle q_k^{n} , x_{n+1} \rangle q_k^n \right ),
\end{equation}

\noindent followed by normalization. The approximated eigenvalues $\{ \lambda_i^{n+1} \}_{i=1}^m$ required for this step are calculated by finding the roots of~\eqref{eqn:TSE}.

This update rule is essentially the iterative rule~\eqref{eqn:general_grad_ext} with the optimal learning rate~\eqref{eqn:opr_eta}, with two exceptions. First, since the true eigenvalues of $X_n^TX_n$ are generally unknown, we replace them by their approximations obtained by finding the roots of~\eqref{eqn:TSE}. Second, similarly to ROIPCA, we introduce the parameter $\mu$. When $\mu = 0$, we obtain the update rule~\eqref{eqn:general_grad_ext} with the optimal learning rate~\eqref{eqn:opr_eta}.

The fROIPCA algorithm is summarized in Algorithm~\ref{alg:onlne_pca}.

\section{Implementation details}  \label{sec:discussion}

\subsection{Time and memory complexity}  \label{sec:complexity}


Based on the analysis given in~\cite{mitz2019symmetric}, the time complexity of each iteration of Algorithm~\ref{alg:onlne_pca} using the eigenvectors formula~\eqref{eqn:TEF} is $O(m^2d)$. The time complexity of the eigenvectors formula~\eqref{eqn:fast_TEF} is $O(md)$. All ROIPCA algorithms need to store the $m$ eigenvectors being updated, requiring $O(md)$ memory. A theoretical comparison of the algorithms' time and space complexity is given in Table~\ref{tbl:complex_comp}.

\subsection{Choosing and updating $\mu$} \label{sec:choose_mu}

We next discuss the effects that the choice of~$\mu$ (see Proposition~\ref{prop:first_order}) has on the time and space complexity of Algorithm~\ref{alg:onlne_pca}. Choosing $\mu = 0$ requires no additional calculations.

Choosing $\mu = \mu_{mean}$ of~\eqref{eq:mu_mean} requires only $\operatorname{tr}(X_n^TX_n)$, and can thus be calculated on-the-fly, since
\begin{equation} \label{eqn:update_mu}
\operatorname{tr}(X_{n+1}^TX_{n+1}) = \operatorname{tr}(X_n^TX_n) + \norm{x_{n+1}}^2.
\end{equation}
Thus, $\mu_{mean}$ can be updated on-the-fly without additional memory.

\section{Numerical results} \label{sec:numerical}

In this section, we support our theoretical claims numerically and compare our algorithms to other online PCA algorithms in terms of accuracy and runtime. We use both synthetic and real-world datasets, as will be described in the following sections. Our datasets are taken from the UCI Machine Learning Repository~\cite{Dua:2019} and the MNIST dataset~\cite{deng2012mnist}, as described in Table~\ref{tbl:data sets}.

The MATLAB code to reproduce the graphs in this section is found in \\ \href{https://github.com/ShkolniskyLab/ROIPCA}{\texttt{github.com/ShkolniskyLab/ROIPCA}}.

\begin{table}
\centering
\begin{tabular}{|l|l|p{10cm}|}
\hline
Name              & Dimension & Description                                                                                                                                                           \\ \hline
MNIST~\cite{deng2012mnist}            & 784       & Each sample is a grey scale image of a handwritten digit between zero and nine.                                                                                        \\ \hline
Superconductivity~\cite{hamidieh2018data} & 81        & Each sample contains 81 features extracted from one of 21263 superconductors.                                                                                                \\ \hline
Poker~\cite{Cattral:2002}             & 10        & Each sample is a hand consisting of five playing cards drawn from a standard deck of 52 cards. Each card is described using two attributes (suit and rank).           \\ \hline
\end{tabular}
\caption{Real-world datasets used.}
\label{tbl:data sets}
\end{table} 

\subsection{Optimality of the learning rate} \label{sec:num_opt}

In this section, we provide numerical evidence for Proposition~\ref{prpo:the_best}. We generate $n = 600$ normally distributed samples with mean zero and a covariance matrix of rank $m=1$, and calculate the first principal component of the first 100 samples. We generate 500 values of $\eta$ that are logarithmically spaced between $10^{-5}$ and $10$. Then, given a new sample, we perform one step of the update rule~\eqref{eqn:general_grad_ext} for each of the values of $\eta$. Denote by $q$ the resulting vector. Denote by $p$ the true principal component of the current $501$ samples. Let $\eta^*$ the optimal learning rate~\eqref{eqn:opr_eta} for this sample. We plot the error $\log_{10}{\norm{p - q}}$ as a function of $\log_{10} \frac{\eta}{\eta^*}$. We repeat this process for all 500 samples. The result is presented in Figure~\ref{fig:proof_opt}. We can see that indeed, a minimum is obtained at $\log_{10} \frac{\eta}{\eta^*} = 0$, as expected.

\begin{figure}
    \centering
        \includegraphics[width=.7\textwidth]{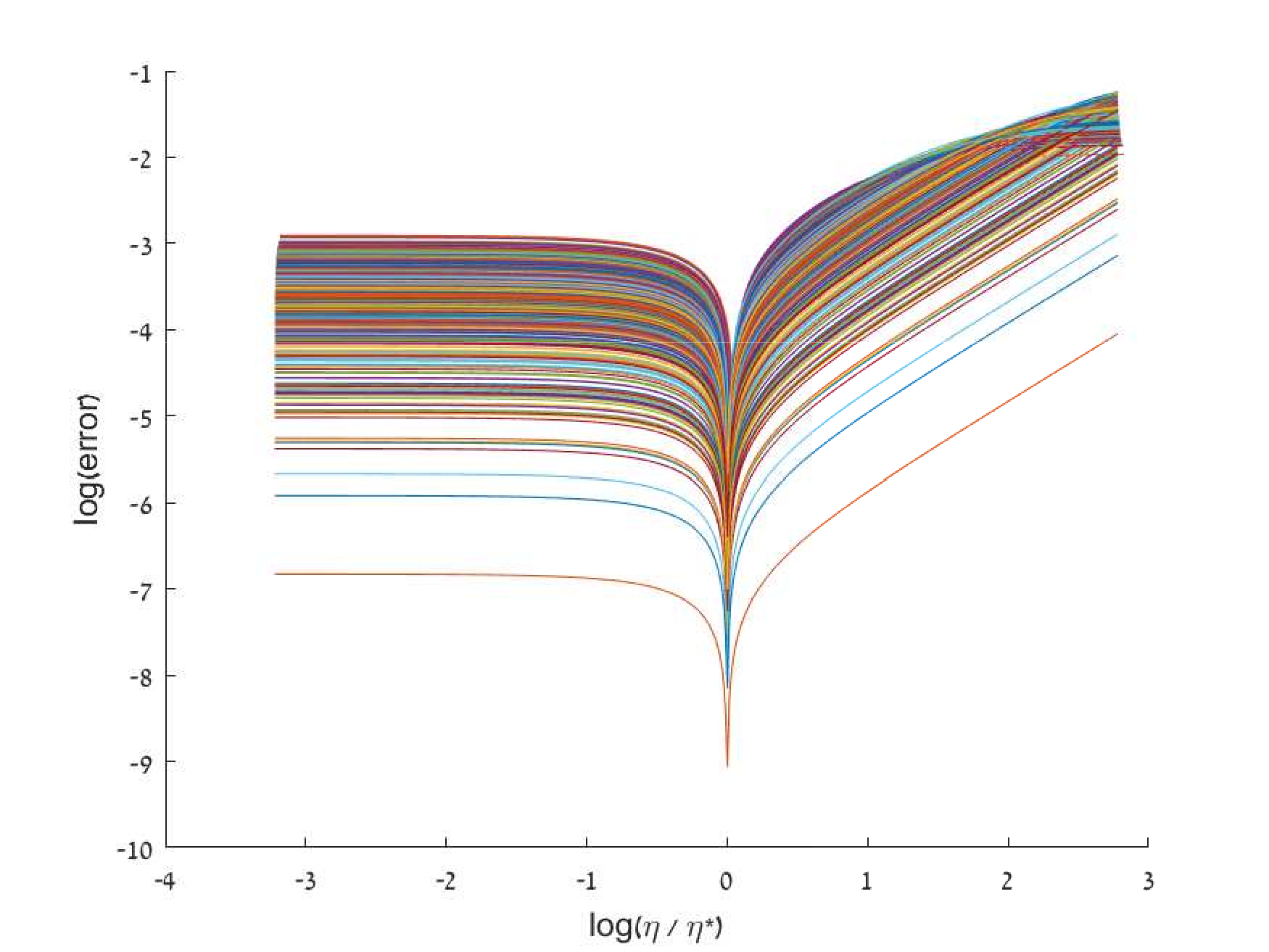}
        \caption{Error after one update as a function of deviation from the optimal learning rate. We can see that indeed, minima are achieved at the optimal learning rate~\eqref{eqn:opr_eta}.} \label{fig:proof_opt}
\end{figure}

\subsection{Accuracy} \label{sec:num_acc}

In this section, we compare several online PCA algorithms on both synthetic and real-world datasets, and demonstrate the superiority of our approach. A comprehensive comparison between online PCA algorithms is given in~\cite{cardot2018online}, which concludes that the method of choice is either IPCA or CCIPCA. We will compare ROIPCA and fROIPCA to our own implementations of IPCA~\cite{arora2012stochastic}, CCIPCA~\cite{weng2003candid}, Oja's rule~\cite{oja1982simplified} and GHA~\cite{sanger1989optimal}. We found the FrequentDirections algorithm~\cite{ghashami2016frequent} not competitive and we will thus not report its results.

We start by comparing the accuracy of all methods. In each experiment, a dataset is randomly generated (the datasets will be described shortly). The tested online PCA algorithms are initialised by a batch PCA of $n_0 = 500$ samples, and are then run on additional $n = 10000$ samples, resulting in an estimated subspace~$\tilde{U}$. We denote by~$U$ the subspace spanned by the $m$ largest eigenvectors of the sample covariance matrix of the entire data set consisting of $n_0 + n$ data points. The parameter $m$ is chosen in each experiment to account for 80\% of the variance of the data, but no more than 10. The metric we use to compare the results of the different algorithms is the relative error in the estimation of~$U$. More precisely, we measure the subspace estimation error by
\begin{equation}\label{eq:error measure}
L(U, \tilde{U}) = \frac{\norm{U^TU - \tilde{U}^T\tilde{U}}_F^2}{\norm{U^TU}_F^2}.
\end{equation}
We repeat this process $20$ times for each dataset, and report the median and standard deviation of the error obtained after the last iteration over all repetitions.

In all of our experiments, we use $\mu = \mu_{mean}$, and the given data are already centered. While all variants of ROIPCA, IPCA and CCIPCA are parameter-free, the learning rate of Oja's rule and GHA algorithms need to be tuned. We use a learning rate of the form $\eta_n = \frac{c}{n}$ with $c$ being determined by a 5-fold cross validation using a separate dataset consisting of $1000$ samples, assuming that the ground truth for this dataset is known.

In our first example, we reproduce the main example given in~\cite{cardot2018online}. In this example, each sample is drawn from a Gaussian distribution with zero mean and covariance matrix $\Gamma = \big(\min (k,l) / d \big)_{1 \leq k,l \leq m}$, with $d \in \{  100, 1000 \}$. This example simulates data with a fast decaying spectrum, meaning that it is practically low-rank. The results are summarized in Table~\ref{tbl:test1}.

\begin{table}
\centering
\begin{tabular}{|l|l|l|}
\hline
           & $d=100$ ($m=1$) & $d=1000$  ($m=1$)  \\ \hline
no update  & 7.8e-3 $\pm$ 1.8e-3                            & 6.2e-4 $\pm$ 2.2e-4                             \\ \hline
ROIPCA     & 3.4e-8 $\pm$ 4.3e-8                            & 8.7e-8 $\pm$ 4.0e-7                             \\ \hline
fROIPCA    & 3.5e-7 $\pm$ 4.4e-8                            & 8.8e-8 $\pm$ 2.6e-7                             \\ \hline
IPCA       & 4.6e-8 $\pm$ 4.5e-8                            & 8.6e-8 $\pm$ 7.5e-7                             \\ \hline
CCIPCA     & 1.3e-4 $\pm$ 1.4e-4                            & 3.3e-4 $\pm$ 2.4e-4                             \\ \hline
GHA        & 1.1e-6 $\pm$ 1.8e-6                            & 1.9e-5 $\pm$ 5.2e-5                             \\ \hline
Oja's rule & 6.3e-6 $\pm$ 6.8e-6                            & 3.5e-5 $\pm$ 1.5e-5                             \\ \hline
\end{tabular}
\caption{Accuracy of the algorithms tested using the example in~\cite{cardot2018online}. We can see that all ROIPCA variants and IPCA outperform all other algorithms in their accuracy. The accuracy of ROIPCA and fROIPCA is comparable.}
\label{tbl:test1}
\end{table}

We next demonstrate our algorithms on real-world datasets. The datasets we use are the MNIST dataset, the superconductivity dataset and the poker dataset, all described in Table~\ref{tbl:data sets}. Table~\ref{tbl:test2} summarizes the error at the final iteration of all experiments.

\begin{table}
\centering
\begin{tabular}{|l|l|l|l|}
\hline
           & MNIST  ($m=10$) & Superconductivity  ($m=2$) & Poker  ($m=5$) \\ \hline
no update  & 2.0e-1 $\pm$ 2.7e-1                           & 2.0e-3 $\pm$ 3.2e-3                                       & 2.2e-3 $\pm$ 4.2e-3                          \\ \hline
ROIPCA     & 4.9e-2 $\pm$ 2.9e-2                           & 1.4e-5 $\pm$ 1.0e-5                                       & 2.5e-7 $\pm$ 8.5e-9                          \\ \hline
fROIPCA    & 5.2e-2 $\pm$ 3.5e-2                           & 1.8e-5 $\pm$ 2.2e-5                                       & 1.0e-6 $\pm$ 1.9e-6                          \\ \hline
IPCA       & 4.9e-2 $\pm$ 2.8e-2                           & 7.4e-6 $\pm$ 1.8e-6                                       & 8.0e-7 $\pm$ 4.7e-7                          \\ \hline
CCIPCA     & 8.1e-2 $\pm$ 1.4e-2                           & 1.7e-4 $\pm$ 2.7e-4                                       & 8.5e-4 $\pm$ 7.2e-4                          \\ \hline
GHA        & 9.8e-2 $\pm$ 1.8e-2                           & 3.6e-3 $\pm$ 2.9e-3                                       & 1.6e-5 $\pm$ 4.7e-5                          \\ \hline
Oja's rule & 1.0e-1 $\pm$ 6.8e-2                           & 5.1e-3 $\pm$ 3.2e-3                                       & 1.6e-5 $\pm$ 2.2e-5                          \\ \hline
\end{tabular}
\caption{Accuracy of the algorithms tested on real-world datasets. We can see that the accuracy of ROIPCA variants and IPCA is superior to other tested algorithms. The accuracy of ROPICA and fROIPCA is comparable.}
\label{tbl:test2}
\end{table}

In our last example, we demonstrate the advantage of the ROIPCA variants for data that are not low-rank. We have seen in the previous examples that the performance of the ROIPCA variants and IPCA is comparable. When the data are not low-rank, this is no longer true due to the parameter $\mu$. As our dataset, we use Gaussian data with mean zero and covariance matrix of dimension $d = 100$ with $m=5$ eigenvalues uniformly distributed between 1 and 1.5. The rest of the eigenvalues are 1. The results are summarized in Table~\ref{tbl:test3}.

\begin{table}
\centering
\begin{tabular}{|l|l|}
\hline
no update   & 3.2e0 $\pm$ 1.3e-1 \\ \hline
ROIPCA  & 2.0e-2 $\pm$ 1.5e-1 \\ \hline
fROIPCA & 2.3e-2 $\pm$ 1.7e-1 \\ \hline
IPCA        & 5.2e-1 $\pm$ 3.6e-1 \\ \hline
CCIPCA      & 2.4e-1 $\pm$ 1.6e-1 \\ \hline
GHA         & 1.3e-1 $\pm$ 1.5e-2 \\ \hline
Oja's rule  & 8.4e-1 $\pm$ 2.3e-1 \\ \hline
\end{tabular}
\caption{Accuracy of the tested algorithms for
data which are not low-rank. We can see the advantage of ROIPCA variants over all other methods.}
\label{tbl:test3}
\end{table}


\subsection{Runtime}

We next compare the runtime of our approach to that of the best performing methods of the previous section, that is to IPCA and CCIPCA. The asymptotic runtime and space complexity of the methods tested is given in Table~\ref{tbl:complex_comp}. Our experiments show that while some methods share the same asymptotic runtime complexity, in practice, there might be  significant differences in their runtimes, depending on the data dimension~$d$. In each experiment, we draw $n = 1000$ samples from a $d$-dimensional Gaussian distribution with mean zero and a diagonal covariance matrix whose 10 largest eigenvalues are distributed uniformly at random between 1 and 2 while the remaining $d - 10$ eigenvalues are zero. The dimension $d$ varies between $100$ and $10,000$. We calculate the top 10 principal components of the entire data, and then update them using 500 additional points drawn from the same distribution, using each of the tested algorithms. We measure the mean runtime of each method per-update. The results are presented in Figure~\ref{fig:runtime_d} on a logarithmic scale.

\begin{figure}
    \centering
        \includegraphics[width=.45\textwidth]{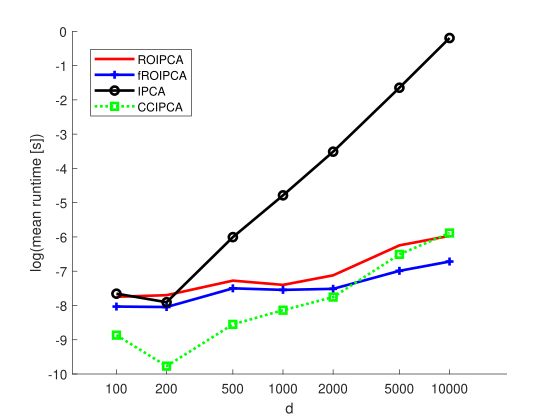}
        \caption{$\log ($mean runtime$)$ of each method as a function of the data dimension $d$. We can see that our algorithms are significantly faster than IPCA when~$d$ is larger than 200. Our fastest variant is fROIPCA, which is comparable in its runtime to CCIPCA.} \label{fig:runtime_d}
\end{figure}

\begin{table}
\centering
\begin{tabular}{|l|l|l|}
\hline
             & Runtime & Space \\ \hline
IPCA~\cite{arora2012stochastic}         & $O(m^2d)$                          & $O(md)$          \\ \hline
CCIPCA~\cite{weng2003candid}       & $O(md)$                            & $O(md)$          \\ \hline
ROIPCA & $O(m^2d)$                          & $O(md)$          \\ \hline
fROIPCA & $O(md)$                          & $O(md)$          \\ \hline
\end{tabular}
\caption{Comparison of runtime and space complexity of the online PCA algorithms. CCIPCA and fROIPCA have superior runtime by a factor of $m$. All algorithms have space complexity of $O(md)$.}
\label{tbl:complex_comp}
\end{table}

\section{Summary}
In this paper, we introduced two online PCA algorithms called ROIPCA and fROIPCA that are based on rank-one updates of the covariance matrix. We proved that our fastest variant, fROIPCA, can be interpreted as a gradient-based method with an optimal learning rate (Propositions~\ref{prpo:the_best} and~\ref{prop:froipca_conv}). We analysed theoretically the error introduced by each variant, and demonstrated numerically that all of our variants are superior in terms of accuracy compared to state-of-the-art methods such as IPCA and CCIPCA. The IPCA algorithm provides accuracy comparable to our algorithms, but is considerably slower for higher dimensions. When the data are not low-rank, our methods were demonstrated to be superior to IPCA as well. The CCIPCA algorithm is the fastest among all methods tested, but it is usually inferior in terms of accuracy. The fROIPCA algorithm is comparable in its runtime to CCIPCA but is significantly more accurate.


Our proposed method can be easily extended to the online kernel PCA setting. In this setting, the updated matrix is the kernel matrix and not the covariance matrix. Introducing a new data point requires adding a row and a column to the kernel matrix, which is a rank-two update~\cite{hoegaerts2007efficiently}. Thus, an update of the principal components of the augmented kernel matrix can be obtained by applying two consecutive rank-one updates.

\section{Proof of Proposition~\ref{prop:fast extension}} \label{app1}

Let $1 \leq i \leq m$ and let $\tau_i \in \mathbb{R}$ be a fixed parameter. We first prove that the term $U^{(m)}(\Delta_i^{(m)})^{-1}(U^{(m)})^Tv$ in the eigenvectors formula~\eqref{eqn:TEF} is approximated by
\begin{equation} \label{eqn:1st_term_app}
\bigg( \frac{1}{\ell_i - \gamma_i} - \tau_i \bigg) \langle u_i , v \rangle u_i + \tau_i (v - r),
\end{equation}
with an error of magnitude $ O \big( \max_{1 \leq j \leq m, j \neq i} \abs{\frac{1}{\ell_j - \gamma_i} - \tau_i} \big) $.

Indeed, the first term in~\eqref{eqn:TEF} can be written as
\begin{align} \label{eq:split_form}
 U^{(m)}(\Delta_i^{(m)})^{-1}(U^{(m)})^Tv &= \sum_{k=1}^{m} \frac{\langle u_k, v \rangle}{\ell_k - \gamma_i}u_k  \notag \\
 &= \frac{\langle u_i, v \rangle}{\ell_i - \gamma_i} + \sum_{k=1 , k\neq i}^{m} \frac{\langle u_k, v \rangle}{\ell_k - \gamma_i}u_k .
\end{align}
Let us replace the terms $\{\frac{1}{\ell_k - \gamma_i}\}_{k=1,k \neq i}^{m}$ by the estimate $\tau_i~\in~\mathbb{R}$ whose optimal value will be determined shortly. We get using~\eqref{eqn:TEF} the following approximation for~\eqref{eq:split_form},
 \begin{align}
& \frac{\langle u_i, v \rangle}{\ell_i - \gamma_i}u_i + \sum_{k=1 , k\neq i}^{m} \tau_i \langle u_k, v \rangle u_k \label{eqn:first_term_app} \\
  &\quad = \frac{\langle u_i, v \rangle}{\ell_i - \gamma_i}u_i + \tau_i \sum_{k=1 , k\neq i}^{m} \langle u_k, v \rangle u_k \notag \\
  &\quad = \frac{\langle u_i, v \rangle}{\ell_i - \gamma_i}u_i + \tau_i \bigg(U^{(m)}(U^{(m)})^T - u_iu_i^T \bigg)v \notag\\
    &\quad = \frac{\langle u_i, v \rangle}{\ell_i - \gamma_i}u_i + \tau_i \bigg(v - r - u_i \langle u_i , v \rangle  \bigg) \notag\\
&\quad = \bigg( \frac{1}{\ell_i - \gamma_i} - \tau_i \bigg)u_i  \langle u_i , v  \rangle + \tau_i (v - r). \notag
\end{align}
Using the orthonormality of the eigenvectors, and recalling that $z_k = \langle u_k , v \rangle $, the squared error of the approximation~\eqref{eqn:first_term_app} is given by
\begin{align}
e^2 &= \norm{\sum_{k=1 , k\neq i}^{m} \frac{\langle u_k, v \rangle}{\ell_k - \gamma_i}u_k - \sum_{k=1 , k\neq i}^{m} \tau_i\langle u_k, v \rangle u_k}^2 \label{eq:tau error}\\
&= \sum_{k=1 , k\neq i}^{m} \bigg( \frac{1}{\ell_k - \gamma_i} - \tau_i \bigg)^2 z_k^2 \label{eqn:to_minimize}\\
& \leq \max_{1 \leq j \leq m, j \neq i} \abs{\frac{1}{\ell_j - \gamma_i} - \tau_i}^2, 
\end{align}
as requested.

A reasonable choice for $\tau_i$ in Proposition~\ref{prop:fast extension} would be the $\tau_i$ that minimizes~\eqref{eqn:to_minimize}. By standard methods we get
\begin{equation} \label{eqn:best_tau2}
\tau^*_i = \frac{\sum_{k=1 , k\neq i}^{m} \frac{z_k^2}{\ell_k - \gamma_i}}{\sum_{k=1 , k\neq i}^{m} z_k^2}.
\end{equation}

\section{Proof of Proposition~\ref{prpo:better}} \label{app2}

For ease of notation, we omit the $n$ superscripts in this proof and denote $t_i^{n+1}$ by $\lambda_i$ for $1 \leq i \leq m$. Let $1 \leq i \leq m$. Following Remark~\ref{rem:scale_not_imp}, we divide $q_i^{n+1}$ in~\eqref{eqn:EigenvaectorFormula} by $\frac{ \langle q_i , x_{n+1} \rangle }{\lambda_i - t_i}$ to get using~\eqref{eq:formula_as_sum}
\begin{equation}\label{eq:pi in proof}
w_i^{n+1} = q_i + \sum_{k=1,k \neq i}^{d} \frac{ (\lambda_i - t_i) \langle q_k , x_{n+1} \rangle }{(\lambda_k - t_i) \langle q_i , x_{n+1} \rangle }q_k .
\end{equation}
Additionally, since
\begin{equation} \label{eqn:general_grad_for_proof}
 x_{n+1} - \sum_{k=1}^{i}  \langle q_k , x_{n+1} \rangle  q_k  = \sum_{k=i+1}^d  \langle q_k , x_{n+1} \rangle  q_k ,
\end{equation}
we have that~\eqref{eqn:general_grad} can be written in the form
\begin{equation} \label{eqn:general_grad_mod}
u_i^1 = q_i + \sum_{k=i+1}^d  \eta_n  \langle q_i , x_{n+1} \rangle  \langle q_k , x_{n+1} \rangle  q_k ,
\end{equation}
and similarly,~\eqref{eqn:general_grad_ext} can be written in the form
\begin{equation} \label{eqn:general_grad_ext_mod}
u_i^2 = q_i
 + \sum_{k=m+1}^d \eta_n  \langle q_i , x_{n+1} \rangle  \langle q_k , x_{n+1} \rangle  q_k .
\end{equation}
It follows using~\eqref{eq:pi in proof} and~\eqref{eqn:general_grad_mod} that
\begin{align*}
e_1^2 &= \norm{w_i^{n+1} - u_i^1}^2  \\
&= \sum_{k=1}^{i - 1} \bigg( \frac{ (\lambda_i - t_i)  \langle q_k , x_{n+1} \rangle }{(\lambda_k - t_i) \langle q_i , x_{n+1} \rangle } \bigg)^2  \\
&\qquad\qquad+ \sum_{k=i+1}^{m} \bigg( \frac{ (\lambda_i - t_i)  \langle q_k , x_{n+1} \rangle }{(\lambda_k - t_i) \langle q_i , x_{n+1} \rangle } - \eta_n  \langle q_i , x_{n+1} \rangle  \langle q_k , x_{n+1} \rangle  \bigg)^2  \\ 
&\qquad\qquad+ \sum_{k=m+1}^{d} \bigg( \frac{ (\lambda_i - t_i)  \langle q_k , x_{n+1} \rangle }{(\lambda_k - t_i) \langle q_i , x_{n+1} \rangle } - \eta_n  \langle q_i , x_{n+1} \rangle  \langle q_k , x_{n+1} \rangle  \bigg)^2 .
\end{align*}
Similarly,
\begin{align*}
e_2^2 &= \norm{w_i^{n+1}  - u_i^2}^2  \\
&= \sum_{k=1}^{i - 1} \bigg( \frac{ (\lambda_i - t_i)  \langle q_k , x_{n+1} \rangle }{(\lambda_k - t_i) \langle q_i , x_{n+1} \rangle } \bigg)^2  \\
&\qquad\qquad+ \sum_{k=i+1}^{m} \bigg( \frac{ (\lambda_i - t_i)  \langle q_k , x_{n+1} \rangle }{(\lambda_k - t_i) \langle q_i , x_{n+1} \rangle } \bigg)^2  \\ 
&\qquad\qquad+ \sum_{k=m+1}^{d} \bigg( \frac{ (\lambda_i - t_i)  \langle q_k , x_{n+1} \rangle }{(\lambda_k - t_i) \langle q_i , x_{n+1} \rangle } - \eta_n  \langle q_i , x_{n+1} \rangle  \langle q_k , x_{n+1} \rangle  \bigg)^2 .
\end{align*}
We conclude that $e_1 \geq e_2$ iff 
\begin{equation}
 \sum_{k=i+1}^{m} \bigg( \frac{ (\lambda_i - t_i)  \langle q_k , x_{n+1} \rangle }{(\lambda_k - t_i) \langle q_i , x_{n+1} \rangle } - \eta_n  \langle q_i , x_{n+1} \rangle \langle q_k , x_{n+1} \rangle  \bigg)^2  
\geq \sum_{k=i+1}^{m} \bigg( \frac{ (\lambda_i - t_i)  \langle q_k , x_{n+1} \rangle }{(\lambda_k - t_i) \langle q_i , x_{n+1} \rangle } \bigg)^2.
\end{equation}

By some manipulations we get that the above inequality is equivalent to
\begin{equation}
\eta_n \bigg( \eta_n \sum_{k=i+1}^{m}  \langle q_i , x_{n+1} \rangle ^2 \langle q_k , x_{n+1} \rangle ^2 
- 2\sum_{k=i+1}^{m} \frac{ (\lambda_i - t_i)  \langle q_k , x_{n+1} \rangle ^2}{(\lambda_k - t_i)} \bigg) \geq 0.
\end{equation}
But since $\eta_n > 0$ we get that $e_1 \geq e_2$ iff 
\begin{equation}\label{eq:eta bound}
\eta_n \geq  \frac{2\sum_{k=i+1}^{m} \frac{ (\lambda_i - t_i)  \langle q_k , x_{n+1} \rangle ^2}{(\lambda_k - t_i)}} { \sum_{k=i+1}^{m}  \langle q_i , x_{n+1} \rangle ^2 \langle q_k , x_{n+1} \rangle ^2  } .
\end{equation}
By the interlacing property~\cite{bunch1978rank}, $t_1 > \lambda_1 > \cdots > t_m > \lambda_m $ and thus $\lambda_i - t_i < 0$ and $\lambda_k - t_i > 0$ for $k > i$, and it follows that the right hand side in~\eqref{eq:eta bound} must be negative. We conclude that $e_1 \geq e_2$ for all $\eta_n > 0$.

\section{Proof of Proposition~\ref{prpo:the_best}} \label{app3}

Let $1 \leq i \leq m$. By~\eqref{eq:formula_as_sum} and Remark~\ref{rem:scale_not_imp}, the exact eigenvector update (up to normalization) is given by
\begin{equation} \label{eq:sum}
\hat{p}_i^{n+1} = \frac{t_i^{n} - t_i^{n+1} }{\langle q_i , x_{n+1} \rangle} \sum_{k=1}^{d} \frac{  \langle q_k , x_{n+1} \rangle }{t_k^{n} - t_i^{n+1}}q_k .
\end{equation}

By the orthogonality of $\{ q_k \}_{k=1}^d$, the vector $q_i^{n+1} \in \text{span} \{ q_k \}_{k \in \{ i , m+1, ..., d \}}$ that minimizes $\norm{\hat{p}_i^{n+1} - q_i^{n+1}}$ is obtained by omitting the first $m$ terms in~\eqref{eq:sum} except for the $i$'th one, that is
\begin{equation*}
q_i^{n+1} = q_i +  \frac{ t_i^{n} - t_i^{n+1}}{ \langle q_i^n , x_{n+1} \rangle } \sum_{k=m+1}^{d} \frac{  \langle q_k^n , x_{n+1} \rangle }{t_k^{n} - t_i^{n+1}}q_k^n.
\end{equation*}

\noindent This is the optimal $q_i^{n+1}$. By the rank assumption, $t_k^{n} = 0$ for $k > m$, and thus,
\begin{equation*}
q_i^{n+1} = q_i^n +  \frac{ t_i^{n+1} - t_i^n}{ \langle q_i^n , x_{n+1} \rangle  t_i^{n+1}} \sum_{k=m+1}^{d}  \langle q_k^n , x_{n+1} \rangle  q_k^n.
\end{equation*}
Since 
\begin{equation} \label{eqn:general_grad_for_proof2}
\sum_{k=m+1}^d  \langle q_k^n , x_{n+1} \rangle  q_k^n  = x_{n+1} - \sum_{k=1}^m  \langle q_k^n , x_{n+1} \rangle  q_k^n,
\end{equation}
we have that
\begin{equation*}
q_i^{n+1} = q_i^n +  \frac{ t_i^{n+1} - t_i^{n}}{ \langle q_i^n , x_{n+1} \rangle  t_i^{n+1}} \Big( x_{n+1} - \sum_{k=1}^m  \langle q_k^n , x_{n+1} \rangle  q_k^n \Big),
\end{equation*}
or equivalently, denoting $\eta_n = \frac{t_i^{n+1} - t_i^{n}}{ \langle q_i^n , x_{n+1} \rangle ^2 t_i^{n+1}}$
\begin{equation} \label{eqn:read_opt}
q_i^{n+1} = q_i^n +  \eta_n \Big(  \langle q_i^n , x_{n+1} \rangle x_{n+1} -  \langle q_i^n , x_{n+1} \rangle  \sum_{k=1}^{m}  \langle q_k^n , x_{n+1} \rangle q_k^n \Big),
\end{equation}
which has the form of~\eqref{eqn:general_grad_ext}. We therefore read the optimal learning rate from~\eqref{eqn:read_opt},  concluding the proof.

\section{Proof of Proposition~\ref{prop:froipca_conv}} \label{ap_conv}

For simplicity, we will prove the case $m = 1$. Denote by $t^{n}$ the largest eigenvalue of $X_n^TX_n$. The proof for the general case is similar. The term $ \frac{1}{ \langle q_n , x_{n+1} \rangle ^2}$ in~\eqref{eqn:opr_eta} is bounded by assumption and hence does not affect the asymptotic behavior of $\eta^*_n$~\eqref{eqn:opr_eta}. Hence, in order to prove that the learning rate~\eqref{eqn:opr_eta} satisfies~\eqref{eqn:suf_cond}, it is sufficient prove that for large enough $N \in \mathbb{N}$, there exist constants $c_1,c_2 > 0$ so that for $n > N$
\begin{equation} \label{eqn:what_we_want}
    \frac{c_1}{n} < 1 - \frac{t^n}{t^{n+1}} = \frac{t^{n+1} - t^n}{t^{n+1}} < \frac{c_2}{n},
\end{equation}
with high probability. 

We note that the denominator in the middle term of~\eqref{eqn:what_we_want} satisfies for all $\ell \in \mathbb{R}$
\begin{equation} \label{eqn:mehane}
    t^{n+1} = (n+1) \Bigg( \frac{t^{n+1}}{n+1} - \ell \Bigg) + (n+1) \ell .
\end{equation}
We will use the following lemma to determine the asymptotic behaviour of the denominator. 

\begin{lemma} \label{lemma:pert_eig}
Let $\{ x_1, x_2, ...\} \in \mathbb{R}^d$ be drawn i.i.d. from some distribution with mean zero and a covariance matrix $\Sigma$. Let $\{  \ell_{i} \}_{i=1}^m $ be the eigenvalues of $\Sigma$ ordered in descending order. Denote by $S_n \in \mathbb{R}^{m \times m}$ the sample covariance matrix of the first $n$ samples, and by $\{ s_i^n \}_{i=1}^{m}$ its eigenvalues. Then, for $1 \leq i \leq m$
\begin{equation}
    s_i^n -\ell_i = O_p \bigg( \frac{1}{\sqrt{n}} \bigg).
\end{equation}
\end{lemma}

\noindent \emph{Proof.} It is known that the sample covariance is a $\sqrt{n}$-consistent estimator for the population covariance~\cite{amemiya1985advanced}, and hence
\begin{equation}
     \norm{\frac{X_n^TX_n}{n} - \Sigma} = O_p \bigg( \frac{1}{\sqrt{n}} \bigg).
\end{equation}
By perturbation theory~\cite{stewart1990matrix},
\begin{equation}
    \abs{s_i^n -\ell_i} \leq \norm{\frac{X_n^TX_n}{n} - \Sigma} = O_p \bigg( \frac{1}{\sqrt{n}} \bigg),
\end{equation}
and hence,
\begin{equation}
    s_i^n -\ell_i = O_p \bigg( \frac{1}{\sqrt{n}} \bigg).
\end{equation}

Back to the proof of the proposition. Since $\frac{t^{n+1}}{n+1}$ is the largest eigenvalue of $\frac{X_{n+1}^TX_{n+1}}{n+1}$,  and by using $\ell$ of~\eqref{eqn:mehane} to be the largest eigenvalue of the the population covariance $\Sigma$, we obtain by Lemma~\ref{lemma:pert_eig} that the term $\big( \frac{t^{n+1}}{n+1} - \ell \big)$ in~\eqref{eqn:mehane} is $O_p\big( \frac{1}{\sqrt{n}}\big)$. Hence $t^{n+1} = O_p(\sqrt{n})+ (n+1) \ell$ which means that with high probability, $t^{n+1} = \Theta(n)$ and hence there exist constants $d_1, d_2 >0$ so that $ d_1 n < t^{n+1} < d_2 n$.

We now determine the asymptotic behaviour of the numerator. Since $t^{n+1}$ is the eigenvalue of the matrix resulting from perturbing the matrix $X_n^TX_n$ by the rank-$1$ perturbation $x_{n+1}x_{n+1}^T$, we get from classical perturbation results that
\begin{equation}
    t^{n+1} = t^n + (q^{n})^T (x_{n+1}x_{n+1}^T) q^n + O \Big( \norm{x_{n+1}x_{n+1}^T}^2 \Big) = t^n +  \langle x_{n+1}, q^n  \rangle^2 + O\Big( \norm{x_{n+1}x_{n+1}^T}^2 \Big).
\end{equation}
Consequently, since $\norm{x_{n+1}x_{n+1}^T}^2 = \norm{x_{n+1}}^4$, we get that the numerator is
\begin{equation} \label{eqn:headache}
    t^{n+1} - t^{n} =  \langle x_{n+1},  q^n  \rangle^2 + O \big( \norm{x_{n+1}}^4 \big).
\end{equation}
By assumption, $\norm{x_{n+1}} = O_p(1)$ and hence $\norm{x_{n+1}}^4 = O_p(1)$. Furthermore, by Cauchy-Schwartz, we have that $\langle x_{n+1},  q^n  \rangle^2 \leq \norm{x_{n+1}}^2 \norm{q^n}^2 = \norm{x_{n+1}}^2 = O_p(1)$. Hence, with high probability, we have that $t^{n+1} - t^{n} < d_3$ for some $d_3 > 0$.
Additionally, since on the right hand side of~\eqref{eqn:headache} we have by assumption $ \langle x_{n+1},  q^n  \rangle^2 > c^2$ , and since $\norm{x_{n+1}}^4 = O_p(1)$ we conclude that with high probability, we have that $t^{n+1} - t^{n} > d_4$ for some $d_4 > 0$.

We combine all of the above to conclude~\eqref{eqn:what_we_want}.

\section*{Funding}
This research was supported by the European Research Council (ERC) under the European Union's Horizon 2020 research and innovation programme (grant
agreement 723991 - CRYOMATH)


\bibliographystyle{plain} 
\bibliography{roipcabib.bib}


\end{document}